\documentclass[review]{elsarticle}

\usepackage{lineno,hyperref}
\modulolinenumbers[5]

\usepackage[ruled, vlined, linesnumbered]{algorithm2e}
\usepackage{lmodern}
\usepackage{amsmath}
\usepackage{mathtools}
\usepackage{amsthm}
\SetKwProg{Init}{init}{}{}
\usepackage{scalerel,stackengine}
\newtheorem{lem}{Lemma}
\newtheorem{theo}{Theorem}
\theoremstyle{definition} \newtheorem{defn}{Definition}
\newcommand{\indepe}{\mathop{\perp\!\!\!\perp}}

\stackMath
\newcommand\reallywidehat[1]{%
\savestack{\tmpbox}{\stretchto{%
  \scaleto{%
    \scalerel*[\widthof{\ensuremath{#1}}]{\kern-.6pt\bigwedge\kern-.6pt}%
    {\rule[-\textheight/2]{1ex}{\textheight}}
  }{\textheight}%
}{0.5ex}}%
\stackon[1pt]{#1}{\tmpbox}%
}
\def\rightarrowCirc{\hbox{$\circ$}\kern-1.5pt\hbox{$\rightarrow$}}
\def\circHyphenCirc{\hbox{$\circ$}\kern-1.5pt\hbox{$-$}\kern-1.5pt\hbox{$\circ$}}
\def\circHyphen{\hbox{$\circ$}\kern-1.5pt\hbox{$-$}}
\newcommand{\argmin}{\mathop{\rm argmin}\limits}

\journal{a journal}

\begin{document}

\begin{frontmatter}

\title{Causal discovery of linear non-Gaussian acyclic models in the presence of latent confounders}

\author[mymainaddress]{Takashi Nicholas Maeda\corref{mycorrespondingauthor}}
\cortext[mycorrespondingauthor]{Corresponding author}

\author[mymainaddress,mysecondaryaddress]{Shohei Shimizu}

\address[mymainaddress]{RIKEN, Tokyo, Japan}
\address[mysecondaryaddress]{Shiga University, Shiga, Japan}

\begin{abstract}
Causal discovery from data affected by latent confounders is an important and difficult challenge. Causal functional model-based approaches have not been used to present variables whose relationships are affected by latent confounders, while some constraint-based methods can present them. This paper proposes a causal functional model-based method called repetitive causal discovery (RCD) to discover the causal structure of observed variables affected by latent confounders. RCD repeats inferring the causal directions between a small number of observed variables and determines whether the relationships are affected by latent confounders. RCD finally produces a causal graph where a bi-directed arrow indicates the pair of variables that have the same latent confounders, and a directed arrow indicates the causal direction of a pair of variables that are not affected by the same latent confounder. The results of experimental validation using simulated data and real-world data confirmed that RCD is effective in identifying latent confounders and causal directions between observed variables.
\end{abstract}

\begin{keyword}
Causal discovery\sep Causal structures \sep Latent confounders\end{keyword}

\end{frontmatter}

\section{Introduction}
Many scientific questions aim to find the causal relationships between variables rather than only find the correlations. While the most effective measure for identifying the causal relationships is controlled experimentation, such experiments are often too costly, unethical, or technically impossible to conduct. Therefore, the development of methods to identify causal relationships from observational data is important.\\
Many algorithms that have been developed for constructing causal graphs assume that there are no latent confounders  (e.g., PC~\cite{Spirtes91}, GES~\cite{chickering2002}, and LiNGAM~\cite{shimizu2006}). They do not work effectively if this assumption is not satisfied. Conversely, FCI~\cite{fci} is an algorithm that presents the pairs of variables that have latent confounders. However, since FCI infers causal relations on the basis of the conditional independence in the joint distribution, it cannot distinguish between the two graphs that entail exactly the same sets of conditional independence. Therefore, to understand the causal relationships of variables where latent confounders exist, we need a new method that satisfies the following criteria: (1)  the method should accurately (without being biased by latent confounders) identify the causal directions between the observed variables that are not affected by latent confounders, and (2)  it should present variables whose relationships are affected by latent confounders.\\
Compared to the constraint-based causal discovery methods (e.g., PC~\cite{Spirtes91} and FCI~\cite{fci}), causal functional model-based approaches~\cite{hoyer2009,Mooij:2009,yamada2010,shimizu2011,peters2014} can identify the entire causal model under proper assumptions. They represent an effect $Y$ as a function of direct cause $X$. They infer that variable $X$ is the cause of variable $Y$ when $X$ is independent of the residual obtained by the regression of $Y$ on $X$ but not independent of $Y$.\\
Most of the existing methods based on causal functional models identify the causal structure of multiple observed variables that form a directed acyclic graph (DAG) under the assumption that there is no latent confounder. They assume that the data generation model is acyclic, and that the external effects of all the observed variables are mutually independent. Such models are called additive noise models (ANMs). Their methods discover the causal structures by the following two steps: (1) identifying the causal order of variables and (2) eliminating unnecessary edges. DirectLiNGAM~\cite{shimizu2011}, which is a variant of LiNGAM~\cite{shimizu2006}, performs regression and independence testing to identify the causal order of multiple variables. DirectLiNGAM finds a {\it root} (a variable that is not affected by other variables) by performing regression and independence testing of each pair of variables. If a variable is exogenous to the other variables, then it is regarded as a root. Thereafter, DirectLiNGAM removes the effect of the root from the other variables and finds the next root in the remaining variables. DirectLiNGAM determines the causal order of variables according to the order of identified roots. RESIT~\cite{peters2014}, a method extended from Mooij et al.~\cite{Mooij:2009} identifies the causal order of variables in a similar manner by performing an iterative procedure. In each step, RESIT finds a {\it sink} (a variable that is not a cause of the other variables). A variable is regarded as a sink when it is endogenous to the other variables. RESIT disregards the identified sinks and finds the next sink in each step. Thus, RESIT finds a causal order of variables. DirectLiNGAM and RESIT then construct a complete DAG, in which each variable pair is connected with the directed edge based on the identified causal order. Thereafter, DirectLiNGAM eliminates unnecessary edges using AdaptiveLasso~\cite{zou}. RESIT eliminates each edge $X\rightarrow Y$ if $X$ is independent of the residual obtained by the regression of $Y$ on $Z/\{X \}$ where $Z$ is the set of causes of $Y$ in the complete DAG.\par
Causal functional model-based methods effectively discover the causal structures of observed variables generated by an additive noise model when there is no latent confounder. However, the results obtained by these methods are likely disturbed when there are latent confounders because they cannot find a causal function between variables affected by the same latent confounders. Furthermore, the causal functional model-based approaches have not been used to show variables that are affected by the same latent confounder, as FCI does. \par
This paper proposes a causal functional model-based method called repetitive causal discovery (RCD) to discover the causal structures of the observed variables that are affected by latent confounders. RCD is aimed at producing causal graphs where a bi-directed arrow indicates the pair of variables that have the same latent confounders, and a directed arrow indicates the direct causal direction between two variables that do not have the same latent confounder. It assumes that the data generation model is linear and acyclic, and that external influences are non-Gaussian. Many causal functional model-based approaches discover causal relations by identifying the causal order of variables and eliminating unnecessary edges. However, RCD discovers the relationships by finding the direct or indirect causes ({\it ancestors}) of each variable, distinguishing direct causes ({\it parents}) from indirect causes, and identifying the pairs of variables that have the same latent confounders.
\par
Our contributions can be summarized as follows:
\begin{itemize}
	\item We developed a causal functional model-based method that can present variable pairs affected by the same latent confounders.
	\item The method can also identify the causal direction of variable pairs that are not affected by latent confounders.
	\item The results of experimental validation using simulated data and real-world data confirmed that RCD is effective in identifying latent confounders and causal directions between observed variables.
\end{itemize} 

A briefer version of this work without detailed proofs can be found in \cite{maeda20a}.

\section{Problem definition}
\label{section:problem}
\subsection{Data generation process}
\label{section:datageneration}
This study aims to analyze the causal relations of observed variables confounded by unobserved variables. We assume that the relationship between each pair of (observed or unobserved) variables is linear, and that the external influence of each (observed or unobserved) variable is non-Gaussian. In addition, we assume that (observed or unobserved) data are generated from a process represented graphically by a directed acyclic graph (DAG). The generation model is formulated using Equation~\ref{equation:lingamco}.
\begin{equation}
\label{equation:lingamco}
	x_{i} = \sum_j b_{ij}x_{j} + \sum_k \lambda_{ik}f_{k}+e_{i}
\end{equation}
where \begin{math}x_{i}\end{math} denotes an observed variable, \begin{math} b_{ij} \end{math} is the causal strength from \begin{math}x_{j}\end{math} to \begin{math}x_{i}\end{math}, \begin{math} f_{k}\end{math} denotes a latent confounder, \begin{math}\lambda_{ik}\end{math} denotes the causal strength from \begin{math}f_{k}\end{math} to \begin{math} x_{i}\end{math}, and \begin{math} e_{i} \end{math} is an external effect. The external effect \begin{math}e_{i}\end{math} and the latent confounder \begin{math}f_{k}\end{math} are assumed to follow non-Gaussian continuous-valued distributions with zero mean and nonzero variance and are mutually independent. The zero/nonzero pattern of \begin{math}b_{ij}\end{math} and \begin{math}\lambda_{ik}\end{math} corresponds to the absence/existence pattern of directed edges. Without loss of generality~\cite{HOYER2008362}, latent confounders \begin{math}f_{k}\end{math} are assumed to be mutually independent. In a matrix form, the model is described as Equation~\ref{equation:lingamcoma}:

\begin{equation}
\label{equation:lingamcoma}
	{\bf x} = {\bf Bx} + {\bf \Lambda f} + {\bf e}
\end{equation}

where the connection strength matrices \begin{math}{\bf B}\end{math} and \begin{math}{\bf \Lambda}\end{math} collect \begin{math}b_{ij}\end{math} and \begin{math}\lambda_{ik}\end{math}, and the vectors \begin{math}{\bf x}\end{math}, \begin{math}{\bf f}\end{math}  and \begin{math}{\bf e}\end{math} collect \begin{math}x_{i}\end{math}, \begin{math}f_{k}\end{math} and \begin{math}e_{i}\end{math}.

\subsection{Research goals}
This study has two goals. First, we extract the pairs of observed variables that are affected by the same latent confounders. This is formulated by \begin{math}{\bf C}\end{math} whose element \begin{math}c_{ij}\end{math} is defined by Equation~\ref{equation:confounderdef}:
\begin{equation}
\label{equation:confounderdef}
	c_{ij} = \begin{cases}
    0 & (\text{if } \forall k, \lambda_{ik}=0 \lor \lambda_{jk}= 0) \\
    1 & (\text{otherwise})
  \end{cases}
\end{equation}
Element \begin{math}c_{ij}\end{math} equals 0 when there is no latent confounder affecting variables \begin{math}x_{i}\end{math} and \begin{math}x_{j}\end{math}. Element \begin{math}c_{ij}\end{math} equals 1 when variables \begin{math}x_{i}\end{math} and \begin{math}x_{j}\end{math} are affected by the same latent confounders.\\
The second goal is to estimate the absence/existence of the causal relations between the observed variables that do not have the same latent confounder. This is defined by a matrix \begin{math}{\bf P}\end{math} whose element \begin{math}p_{ij}\end{math} is expressed by Equation~\ref{equation:causalrelationdef}:
\begin{equation}
\label{equation:causalrelationdef}
	p_{ij} = \begin{cases}
    0 & (\text{if } b_{ij} = 0 \text{ or } c_{ij} = 1) \\
    1 & (\text{otherwise})
  \end{cases}
\end{equation}
$p_{ij}=0$ when $c_{ij}=1$ because we do not aim to identify the causal direction between the observed variables that are affected by the same latent confounders.\par

Finally, RCD produces a causal graph where a bi-directed arrow indicates the pair of variables that have the same latent confounders, and a directed arrow indicates the causal direction of a pair of variables that are not affected by the same latent confounder. For example, assume that using the data generation model shown in Figure~\ref{figure:intro}-(a), our final goal is to draw a causal diagram shown in Figure~\ref{figure:intro}-(b), where variables $f_1$ and $f_2$ are latent confounders, and variables A--H are observed variables.
\begin{figure}[t]
\centering
\includegraphics[width=8.0cm]{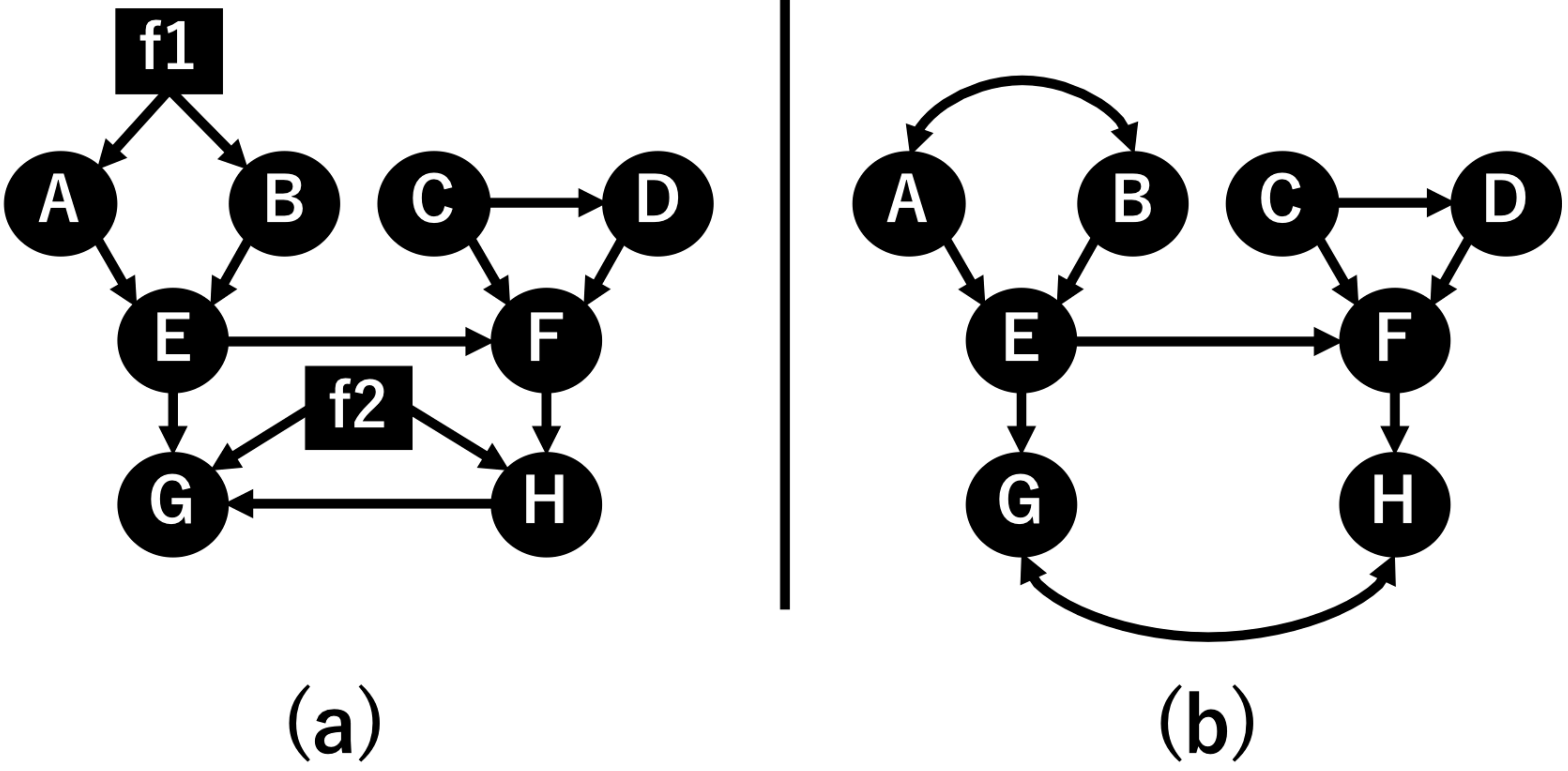}
\caption{(a) Data generation model ($f_1$ and $f_2$ are latent confounders). (b) Causal graph that RCD produces. A bi-directed arrow indicates that two variables are affected by the same latent confounders.}
\label{figure:intro}
\end{figure}

\section{Proposed Method}
\label{section:proposedmethod}

\subsection{The framework}
RCD involves three steps: (1) It extracts a set of {\it ancestors} of each variable. {\it Ancestor} is a direct or indirect cause. In this paper, $M_i$ denotes the set of ancestors of $x_i$. $M_i$ is initialized as $M_i = \emptyset$. RCD repeats the inference of causal directions between variables and updates $M$. When inferring the causal directions between observed variables, RCD removes the effect of the already identified common ancestors. Causal direction between variables $x_i$ and $x_j$ can be identified when the set of identified common causes (i.e. $M_i \cap M_j$) satisfies the back-door criterion~\cite{Pearl1993, pearl2000} to $x_i$ and $x_j$. The repetition of causal inference is stopped when $M$ no longer changes. (2) RCD extracts {\it parents} (direct causes) from $M$. When $x_j$ is an ancestor but not a parent of $x_i$, the causal effect of $x_{j}$ on $x_i$ is mediated through $M_i \setminus \{x_k \}$. RCD distinguishes direct causes from indirect causes by inferring conditional independence. (3) RCD finds the pairs of variables that are affected by the same latent confounders by extracting the pairs of variables that remain correlated but whose causal direction is not identified.

\subsection{Finding ancestors of each variable}
RCD repeats the inference of causal directions between a given number of variables to extract the ancestors of each observed variable. We introduce Lemmas 1 and 2, by which the ancestors of each variable can be identified when there is no latent confounder. Then, we extend them to Lemma 3 by which RCD extracts the ancestors of each observed variable for the case that latent confounders exist. We first quote Darmois-Skitovitch theorem (Theorem~\ref{theo:darmo}) proved in \cite{Darmois,Skitovitch} because it is used to prove the lemmas.

\begin{theo}
\label{theo:darmo}
Define two random variables $y_1$ and $y_2$ as linear combinations of independent random variables $s_{i} (i=1, \cdots, q)$: $Y_1=\sum_{i=1}^{q} \alpha_i s_i$, $Y_2=\sum_{i=1}^{q} \beta_i s_i$. Then, if $y_1$ and $y_2$ are independent, all variables $s_j$ for which $\alpha_{j}\beta_{j}\neq 0$ are Gaussian. In other words, if there exists a non-Gaussian $s_j$ for which $\alpha_j \beta_j \neq 0$, $y_1$ and
$y_2$ are dependent.
\end{theo}

\begin{lem}
Assume that there are variables $x_i$ and $x_j$, and their causal relation is linear, and their external influences $e_i$ and $e_j$ are non-Gaussian and mutually independent. Let $r_i^{(j)}$ denote the residual obtained by the linear regression of $x_i$ on $x_j$ and $r_j^{(i)}$ denote the residual obtained by the linear regression of $x_j$ on $x_i$. The causal relation between variables $x_i$ and $x_{j}$ is determined as follows: (1) If $x_i$ and $x_j$ are not linearly correlated, then there is no causal effect between $x_i$ and $x_j$. (2) If $x_i$ and $x_j$ are linearly correlated and $x_{j}$ is independent of residual $r_i^{(j)}$, then $x_{j}$ is an ancestor of $x_{i}$. (3) If $x_i$ and $x_j$ are linearly correlated and $x_{j}$ is dependent on $r_i^{(j)}$ and $x_{i}$ is dependent on $r_j^{(i)}$, then $x_i$ and $x_{j}$ have a common ancestor. (4) There is no case that $x_i$ and $x_j$ are linearly correlated and $x_{j}$ is independent of $r_i^{(j)}$ and $x_{i}$ is independent of $r_j^{(i)}$.
\end{lem}

\begin{proof}
The causal relationship between two variables $x_i$ and $x_j$ can be classified into the following four cases: (Case 1) There is no common cause of the two variables, and there is no causal effect between them; (Case 2) There is no common cause of the two variables, and one variable is a cause of the other variable; (Case 3) There are common causes of the two variables, and there is no causal effect between them; (Case 4) There are common causes of the two variables, and one variable is a cause of the other variable. Cases 1, 2, 3, and 4 are modeled by Equations \ref{equation:model1-1}, \ref{equation:model1-2}, \ref{equation:model1-3}, and \ref{equation:model1-4}, respectively:
\begin{alignat}{4}
		&x_i = e_i, \ \ \ \ \ & x_j &= e_j\label{equation:model1-1}\\
		&x_i = b_{ij} x_j + e_i, \ \ \ \ \ & x_j &= e_j\label{equation:model1-2}\\
		&x_i = c_i + e_i, \ \ \ \ \ & x_j &= c_j + e_j\label{equation:model1-3}\\
		&x_i = b_{ij} x_j + c_i + e_i, \ \ \ \ \ & x_j &= c_j + e_j\label{equation:model1-4}
\end{alignat}
where $e_i$ and $e_j$ are the non-Gaussian external effects that are mutually independent, $b_{ij}$ is the non-zero causal strength from $x_j$ to $x_i$, and $c_i$ and $c_j$ are the linear combinations of the common causes of $x_i$ and $x_j$. The linear combinations of the common causes $c_i$ and $c_j$ are linearly correlated and are independent of $e_i$ and $e_j$.
We investigate the following three points for each case: (1) whether $x_i$ and $x_j$ are linearly correlated, (2) whether $x_j$ is independent of $r_i^{(j)}$, and (3) whether $x_i$ is independent of $r_j^{(i)}$.\\
\noindent
\underline{\bf Case 1:} Variables $x_i$ and $x_j$ are mutually independent because of Equation~\ref{equation:model1-1}. Therefore, $x_i$ and $x_j$ are not linearly correlated. Let $\alpha$ denote the coefficient of $x_j$ when $x_i$ is regressed on $x_j$. Since $x_i$ and $x_j$ are mutually independent, $\alpha=0$. Then,
\begin{align}
\label{equation:case2-1}
	r_i^{(j)} &= x_i - \alpha x_j\nonumber\\
	&=x_i
\end{align}
Therefore, $x_j$ is independent of $r_i^{(j)}$ because $x_i$ and $x_j$ are mutually independent. Similarly, $x_i$ is independent of $r_j^{(i)}$.\\
\noindent
\underline{\bf Case 2:} Variables $x_i$ and $x_j$ are linearly correlated because $x_i=b_{ij} x_j + e_i$. Let $\alpha$ denote the coefficient of $x_j$ when $x_i$ is regressed on $x_j$. Then, $\alpha = b_{ij}$ because $b_{ij}x_j$ is the only term on the right side of equation $x_i=b_{ij}x_j+e_i$ that covaries with $x_j$. Then, we have $r_i^{(j)}$:
\begin{align}
\label{equation:case2-2}
	r_i^{(j)} &= x_i - \alpha x_j \nonumber \\
	&=b_{ij} x_j + e_i - \alpha x_j \nonumber \\
	&=e_{i}
\end{align}
Then, $x_j$ is independent of $r_i^{(j)}$ because $x_j$ is independent of $e_i$. Let $\beta$ denote the coefficient of $x_i$ when $x_j$ is regressed on $x_i$. Since $x_i$ and $x_j$ are linearly correlated, $\beta \neq 0$. Then, we have $r_j^{(i)}$:
\begin{align}
\label{equation:case2-3}
	r_j^{(i)} &= x_j - \beta x_i \nonumber \\
	&=x_j - \beta \left(b_{ij} x_j + e_i \right)\nonumber \\
	&= \left( 1- b_{ij}\beta \right)x_j - \beta e_{i}\nonumber \\
	&= \left( 1- b_{ij}\beta \right)e_j - \beta e_{i}
\end{align}
Then, $x_i$ is not independent of $r_j^{(i)}$ because of the term $-\beta e_{i}$ in Equation~\ref{equation:case2-3} and Theorem~\ref{theo:darmo}.\\
\noindent
\underline{\bf Case 3:} Since $c_i$ and $c_j$ are linearly correlated, $x_i$ and $x_j$ are linearly correlated. Let $\alpha$ denote the coefficient of $x_j$ when $x_i$ is regressed on $x_j$. Since $x_i$ and $x_j$ are linearly correlated, $\alpha \neq 0$. Then, we have $r_i^{(j)}$:
\begin{align}
\label{equation:case2-5}
	r_i^{(j)} &= x_i - \alpha x_j \nonumber \\
	&=c_i + e_i - \alpha \left(c_j + e_j\right)\nonumber \\
	&= c_i + e_i - \alpha c_j - \alpha e_j
\end{align}
Then, $x_j$ is not independent of $r_i^{(j)}$ because of the term $-\alpha e_j$ in Equation~\ref{equation:case2-5} and Theorem~\ref{theo:darmo}. Similarly, $x_i$ is not independent of $r_j^{(i)}$.\\
\noindent
\underline{\bf Case 4:} 
Since $c_i$ and $c_j$ are linearly correlated, $x_i$ and $x_j$ are linearly correlated. Let $\alpha$ denote the coefficient of $x_j$ when $x_i$ is regressed on $x_j$. Then, $\alpha \neq b_{ij}$ because $x_j$ covaries with terms $b_{ij}x_j$ and $c_i$ on the right side of equation $x_i = b_{ij} x_j + c_i + e_i$. We have $r_i^{(j)}$:
\begin{align}
\label{equation:case2-6}
	r_i^{(j)} &= x_i - \alpha x_j \nonumber \\
	&=b_{ij} x_j + c_i + e_i- \alpha \left(c_j + e_j\right)\nonumber \\
	&= b_{ij} \left(c_j + e_j\right) + c_i + e_i- \alpha \left(c_j + e_j\right)\nonumber\\
	&= \left( b_{ij} - \alpha \right) c_j + \left( b_{ij} - \alpha \right) e_j+ c_i + e_i
\end{align}
Then, $x_j$ is not independent of $r_i^{(j)}$ because of the term $\left( b_{ij} - \alpha \right) e_j$ in Equation~\ref{equation:case2-5} and Theorem~\ref{theo:darmo}. Let $\beta$ denote the coefficient of $x_i$ when $x_j$ is regressed on $x_i$. Since $x_i$ and $x_j$ are linearly correlated, $\beta \neq 0$. Then, we have $r_j^{(i)}$:
\begin{align}
\label{equation:case2-7}
	r_j^{(i)} &= x_j - \beta x_i \nonumber \\
	&=x_j - \beta \left(b_{ij} x_j + c_i + e_i \right)\nonumber \\
	&= \left( 1- b_{ij}\beta \right)x_j - \beta c_{i} - \beta e_{i}\nonumber \\
	&= \left( 1- b_{ij}\beta \right)(c_j + e_j) - \beta c_{i} - \beta e_{i}
\end{align}
Then, $x_i$ is not independent of $r_j^{(i)}$ because of the term $-\beta e_{i}$ in Equation~\ref{equation:case2-7} and Theorem~\ref{theo:darmo}. These cases can be summarized as follows: (Case 1) $x_i$ and $x_j$ are not linearly correlated; (Case 2) $x_i$ and $x_j$ are linearly correlated, $x_j$ is independent of $r_i^{(j)}$, and $x_i$ is not independent of $r_j^{(i)}$ when the causal direction is $x_i \leftarrow x_j$; (Cases 3 and 4) $x_i$ and $x_j$ are linearly correlated, $x_j$ is not independent of $r_i^{(j)}$, and $x_i$ is not independent of $r_j^{(i)}$.
Lemma 1-(1) assumes that $x_i$ and $x_j$ are not linearly correlated. This assumption only corresponds to Case 1. Therefore, there is no causal effect between $x_i$ and $x_j$. Lemma 1-(2) assumes that $x_i$ and $x_j$ are linearly correlated, and $x_j$ is independent of $r_i^{(j)}$. This assumption only corresponds to Case 2. Therefore, $x_j$ is an ancestor of $x_i$. Lemma 1-(3) assumes that $x_i$ and $x_j$ are linearly correlated, $x_j$ is not independent of $r_i^{(j)}$, and $x_i$ is not independent of $r_j^{(i)}$. This corresponds to Case 3 or Case 4. Therefore, $x_i$ and $x_j$ have common ancestors. According to Lemma 1-(4), there is no case among Cases 1--4 where $x_i$ and $x_j$ are linearly correlated, $x_j$ is independent of $r_i^{(j)}$, and $x_i$ is independent of $r_j^{(i)}$.
\end{proof}

It is necessary to remove the effect of common causes to infer the causal directions between variables. When the set of the identified common causes of variables $x_i$ and $x_j$ satisfies the back-door criterion, the causal direction between $x_i$ and $x_j$ can be identified. The back-door criterion~\cite{Pearl1993, pearl2000} is defined as follows:

\begin{defn}
A set of variables $Z$ satisfies the back-door criterion relative to an ordered pair of variables ($x_i$, $x_j$) in a DAG $G$ if no node in $Z$ is a descendant of $x_i$, and $Z$ blocks every path between $x_i$ and $x_j$ that contains an arrow into $x_i$.
\end{defn}

Lemma 1 is generalized to Lemma 2 to incorporate the process of removing the effects of the identified common causes. Lemma 2 can also be used to determine whether the identified common causes are sufficient to detect the causal direction between the two variables.

\begin{lem}
Let $H_{ij}$ denote the set of common ancestors of $x_i$ and $x_j$. Let $y_i$ and $y_j$ denote the residuals when $x_{i}$ and $x_j$ are regressed on $H_{ij}$, respectively. Let $r_i^{(j)}$ and $r_j^{(i)}$ denote the residual obtained by the linear regression of $y_i$ on $y_j$, and $y_j$ on $y_i$, respectively. The causality and the existence of the confounders are determined by the following criteria: (1) If $y_i$ and $y_j$ are not linearly correlated, then there is no causal effect between $x_i$ and $x_j$. (2) If $y_i$ and $y_j$ are linearly correlated and $y_{j}$ is independent of the residual $r_i^{(j)}$, then $x_{j}$ is an ancestor of $x_{i}$. (3) If $y_i$ and $y_j$ are linearly correlated and $y_{j}$ is dependent on $r_i^{(j)}$ and $y_{i}$ is dependent on $r_j^{(i)}$, then $x_i$ and $x_{j}$ have a common ancestor other than $H_{ij}$, and $H_{ij}$ does not satisfy the back-door criterion to $(x_i, x_j)$ or $(x_j, x_i)$. (4) There is no case that $y_i$ and $y_j$ are linearly correlated and $y_{j}$ is independent of $r_i^{(j)}$ and $y_{i}$ is independent of $r_j^{(i)}$.
\end{lem}
\begin{proof}
	When Lemma 1 is applied to $y_i$ and $y_j$, Lemma 2 is derived.
\end{proof}

\begin{figure}[t]
\centering
\includegraphics[width=6.5cm]{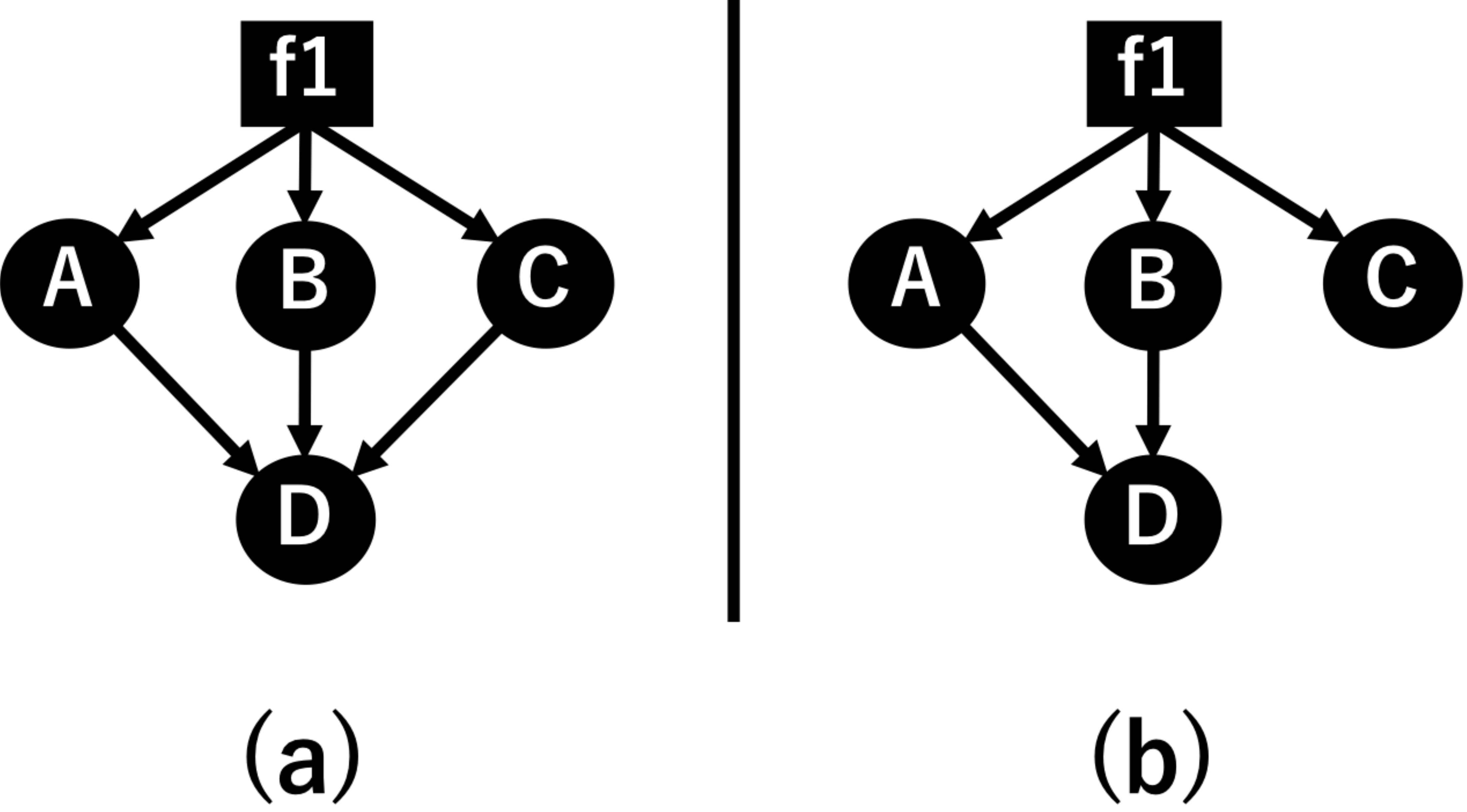}
\caption{(a) Variables $A$, $B$, and $C$ are the causes of variable $D$, and they have a common cause, $f_1$. (b) $A$ and $B$ are the causes of $D$,  but $C$ is not.}
\label{figure:confounder}
\end{figure}

Next, we consider the case that there are latent confounders. In Lemma 2, the direction between two variables is inferred by regression and independence tests. However, if there are two paths from latent confounder $f_k$ to $x_i$, and $x_j$ is only on one of the paths, then $M_i \cap M_j$ cannot satisfy the back-door criterion. For example, in Figure~\ref{figure:confounder}-(a), variables A, B, and C are the causes of variable D, and the causes are also affected by the same latent confounder $f_1$. The causal direction between $A$ and $D$ cannot be inferred only by inferring the causality between them because the effect of $f_1$ is mediated through $B$ and $C$ to $D$. Therefore, $A$, $B$, and $C$ are the causes of $D$ when they are independent of the residual obtained by the multiple regression of $D$ on $\{ A, B, C \}$. However, it is necessary to confirm that variables in each proper subset of $\{A, B, C\}$ are not independent of the residual obtained by the regression of $D$ on the proper subset (i.e., no proper subset of $\{A, B, C\}$ satisfies the back-door criterion). For example, in Figure~\ref{figure:confounder}-(b), $C$ is not a cause of $D$, but $A$, $B$, and $C$ are all independent of the residual obtained by the multiple regression of $D$ on $\{ A, B, C \}$. $C$ should not be regarded as a cause of $D$ because $A$ and $B$ are also independent of the residual when $D$ is regressed on $\{A, B\}$. This example is generalized and formulated by Lemma 3:

\begin{lem}
Let $X$ denote the set of all observed variables. Let $U$ denote a subset of $X$ that contains $x_i$ (i.e., $U \subseteq X$ and $x_i \in U$). Let $M$ denote the sequence of $M_j$ where $M_j$ is a set of ancestors of $x_j$. For each $x_j\in U$, let $y_j$ denote the residual obtained by the multiple linear regression of $x_j$ on the common ancestors of $U$, where the set of common ancestors of $U$ is $\bigcap_{x_j\in U}M_j$. We define $f(x_i, U, M)$ as a function that returns 1 when each $y_j\in\{y_j \mid x_j \in U\setminus x_i \}$ is independent of the residual obtained by the multiple linear regression of $y_i$ on $\{y_j \mid j \neq i\}$; otherwise it returns 0. If $f(x_i, V, M)=0$ for each $V \subset U$ and $f(x_i, U, M)=1$, then each $x_j \in U$ is an ancestor of $x_j$.
\end{lem}

\begin{proof}
	We prove Lemma 3 by contradiction. Assume that $x_j \in U \setminus \{x_i\}$ is not an ancestor of $x_i$, even though $f(x_i, V, M)=0$ for each $V \subset U$, and $f(x_i, U, M)=1$. Let $D_j$ denote the set that consists of the descendants of $x_j$ and $x_j$ itself. Then,
\begin{equation*}
\label{equation:a3-1}
	x_{i} = \sum_{x_m\notin D_j} b_{im}x_{m} + \sum_n \lambda_{in}f_{n}+e_{i}
\end{equation*}
Let $H_U$ denote the set of common causes of $U$ (i.e. $H_U = \bigcap_{x_j\in U}M_j$). Let $\alpha_k$ denote the coefficient of $x_k\in H_U$ when $x_i$ is regressed on $H_U$. Then, 
\begin{equation*}
\label{equation:a3-2}
	y_i = x_i -\sum_{x_k\in H_{U}} \alpha_k x_{k}
\end{equation*}
Let $s_i^{U}$ denote the residual obtained by the multiple regression of $y_i$ on $\{y_j \mid x_j \in U\setminus x_i \}$, and let $\beta_k$ denote the coefficient of $y_k$ obtained by the multiple regression of $y_i$ on $y_k\in\{y_k \mid x_k \in U\setminus \{x_i\} \}$. Then, we have $s_i^U$:
\begin{align}
\label{equation:a3-3}
	s_i^U &= y_i - \sum_{x_k \in U\setminus \{x_i\}}\beta_k y_k \nonumber\\
	&= y_i - \beta_j y_j - \sum_{x_k \in U\setminus \{x_i, x_j\}}\beta_k y_k\nonumber\\
	&=x_i -\sum_{x_k\in H_{U}} \alpha_k x_{k}- \beta_j y_j - \sum_{x_k \in U\setminus \{x_i, x_j\}}\beta_k y_k\nonumber \\
	&=\sum_{x_m\notin D_j} b_{im}x_{m} + \sum_n \lambda_{in}f_{n}+e_{i}-\sum_{x_k\in H_{U}} \alpha_k x_{k}- \beta_j y_j - \sum_{x_k \in U\setminus \{x_i, x_j\}}\beta_k y_k
\end{align}
There is no term that includes $e_j$, the external effect of $y_j$, other than $-\beta_j y_j$ in Equation~\ref{equation:a3-3}. External effect $e_j$ is independent of the other terms in Equation~\ref{equation:a3-3}. Since $y_j$ is independent of $s_i^U$, $\beta_j=0$ by Theorem~\ref{theo:darmo}. Therefore, we have $s_i^U$ as follows:
\begin{align}
\label{equation:a3-4}
	s_i^U &= y_i - \sum_{x_k \in U\setminus \{x_i,x_j\}}\beta_k y_k 
\end{align}
Every $y_k \in U\setminus\{x_i,x_j\}$ is independent of $s_i^U$. This means $f(x_i, U\setminus \{x_j\}, M)=1$, and it contradicts the assumption; that is, $f(x_i, V, M)=0$ for each $V \subset U$.
\end{proof}

We describe the procedure and the implementation of how RCD extracts the ancestors of each observed variable in Algorithm~\ref{algo:ancestoralgorithm}. The output of the algorithm is sequence $M=\{M_i \}$, where $M_i$ is the set of identified ancestors of $x_i$. Argument $\alpha_C$ is the alpha level for the p-value of the Pearson's correlation. If the p-value of two variables is smaller than $\alpha_{\text{C}}$, then we estimate that the variables are linearly correlated. Argument $\alpha_{\text{I}}$ is the alpha level for the p-value of the Hilbert-Schmidt independence criterion (HSIC)~\cite{NIPS2007_3201}. If the p-value of the HSIC of two variables is greater than $\alpha_{\text{I}}$, then we estimate that the variables are mutually independent. Argument $\alpha_S$ is the alpha level to test whether a variable is generated from a non-Gaussian process using the Shapiro-Wilk test~\cite{shapiro1965}. Argument $n$ is the maximum number of explanatory variables used in multiple linear regression for identifying causal directions; i.e., the maximum number of $(|U|-1)$ in Lemma 3. In practice, this should be set to a small number when the number of samples is smaller than the number of variables. RCD does not perform multiple regression analysis of more than $n$ explanatory variables.\par
RCD initializes $M_i$ to be an empty set for each $x_i \in X$. RCD repeats the inference between the variables in each $U \subset X$ that has $(l+1)$ elements. Number $l$ is initialized to $1$. If there is no change in $M$, $l$ is increased by 1. If there is a change in $M$, $l$ is set to $1$. When $l$ exceeds $n$, the repetition ends. Variable $changed$ has information about whether there is a change in $M$ within an iteration.\par
In line 16 of Algorithm~\ref{algo:ancestoralgorithm}, RCD confirms that there is no identified ancestor of $x_i$ in $U$ by checking that $M_i\cap U=\emptyset$. This confirms that $f(x_i, V, M)=0$ for each $V \subset U$ in Lemma 3. In lines 17--24, RCD checks whether $f(x_i, U, M)=1$ in Lemma 3. When $f(x_i, U, M)=1$ is satisfied, $x_i$ is put into $S$. $S$ is a set of candidates for a {\it sink} (a variable that is not a cause of the others) in $U$. It is necessary to test whether there is only one {\it sink} in $U$ because two variables may be misinterpreted as causes of each other when the alpha level for the independence test ($\alpha_{\text I}$) is too small.\par
We use least squares regression for removing the effect of common causes in line 12 of Algorithm~\ref{algo:ancestoralgorithm}, but we use a variant of multiple linear regression called multilinear HSIC regression (MLHSICR) to examine the causal directions between variables in $U$ in line 20 of Algorithm~\ref{algo:ancestoralgorithm} when $l \geq 2$. Coefficients obtained by multiple linear regression using the ordinary least squares method with linearly correlated explanatory variables often differ from true values due to estimation errors. Thus, the relationship between the explanatory variables and the residual may be misinterpreted to be dependent in the case that explanatory variables are affected by the same latent confounders. To avoid such failure, we use MLHSICR defined as follows:

\begin{defn}
Let variable $x_i$ denote an explanatory variable, ${\bf x}$ denote a vector that collects explanatory variables $x_i$, and $y$ denote a response variable. MLHSICR models the relationship $y={\boldsymbol \lambda}^{\top}{\bf x}$ by the coefficient vector $\boldsymbol{\lambda}$ in the following equation:
\begin{equation}
\label{eq:hsic}
	{\boldsymbol \lambda} = \argmin_{\boldsymbol \lambda}\sum_i \reallywidehat{\text{HSIC}}(x_i,y - {\boldsymbol \lambda}^{\top}{\bf x})
\end{equation}
where ${\reallywidehat{\text{HSIC}}}(a, b)$ denotes the Hilbert-Schmidt independence criterion of $a$ and $b$.
\end{defn}

Mooij et al.~\cite{Mooij:2009} have developed a method to estimate the nonlinear causal function between variables by minimizing the HSIC between the explanatory variables and the residual. RCD estimates ${\boldsymbol \lambda}$ by minimizing the sum of the HSICs in Equation~\ref{eq:hsic} using the L-BFGS method~\cite{Liu1989}, similar to Mooij et al.~\cite{Mooij:2009}. L-BFGS is a quasi-Newton method, and RCD sets the coefficients obtained by the least squares method to the initial value of ${\boldsymbol \lambda}$.

\begin{algorithm}
\scriptsize
\SetKwProg{init}{initialization}{}{}
\DontPrintSemicolon
\KwIn{$X$: the set of observed variables, $\alpha_{\text{C}}$: the alpha level for Pearson's correlation, $\alpha_{\text{I}}$: the alpha level for independence test, $\alpha_{\text{S}}$: the alpha level for Shapiro-Wilk test, $n$: the maximum number of explanatory variables}
\KwOut{$M$: the sequence $\{ M_i\}$ where $M_{i}$ is a set of ancestors of $x_i$.}
\SetKwBlock{Begin}{function}{end function}
\Begin($\text{extractAncestors} {(} X, \alpha_{\text{C}}, \alpha_{\text{I}}, \alpha_{\text{S}}, n {)}$)
{     
    \init{}{
		\ForEach{$i$}{
			$M_i \leftarrow \emptyset$\;
		}
		$l \leftarrow 1$\;
	}
	\While{$l \leq n$}{
	$changed \leftarrow$ FALSE\;
		\ForEach{$U \subseteq X; (|U|=l+1)$}{
			$H_U \leftarrow \bigcap_{x_j\in U}M_j$\;
			$S \leftarrow \emptyset$\;
			\ForEach{$x_j \in U$}{
				$y_j \leftarrow$ the residual obtained by regression of $x_j$ on $H_U$\;
				$t_j  \leftarrow$ the p-value of Shapiro-Wilk test of $y_j$\;
			}
			\If{$\forall t_{k} < \alpha_{\text{S}}$}{
			\ForEach{$x_i \in U$}{
				\If{$M_i \cap U = \emptyset$}{
					\ForEach{$x_{j} \in U \setminus \{x_i\}$}{
						$c_{ij}  \leftarrow$ the p-value of linear correlation between $y_i $ and $y_{j}$\;
					}
					\If{$\forall c_{ij} < \alpha_{\text{C}}$}{
						$s_i^{U} \leftarrow$ the residual obtained by regression of $y_i$ on $\{y_j | x_j \in U\setminus \{x_i\} \}$\;
						\ForEach{$x_{j} \in U \setminus \{x_i\}$}{
							$h_{ij}  \leftarrow$ the p-value of the HSIC between $s_i^{U}$ and $y_{j}$\;
						}
						\If{$\forall h_{ij} > \alpha_{\text{I}}$}{
							$S \leftarrow S\cup \{x_i\}$\;
						}
					}
				}
			}
			\If{$|S|=1$}{
				\ForEach{$x_{i} \in S$}{
					$M_{i} \leftarrow M_{i} \cup (U\setminus \{x_i \})$\;
				}
				$changed \leftarrow$ TRUE\;
			}
		}
	}
	\uIf{$changed=\rm{TRUE}$}{
		$l \leftarrow 1$\;
	}\Else{
		$l \leftarrow l + 1$\;
	}
	}
  \Return{$M$}
}
\caption{Extract ancestors of each variable}\label{algo:ancestoralgorithm}
\end{algorithm}

\subsection{Finding parents of each variable}

When $x_j$ is an ancestor but not a parent of $x_i$, the effect of $x_j$ on $x_i$ is mediated through $M_i \setminus \{ x_j \}$. Therefore, $x_j \indepe x_i \mid M_i \setminus \{ x_j \}$. \cite{zhang2017} proposed a method to test the conditional independence using unconditional independence testing in Theorem~\ref{theo:zhang} (proved by them):
\begin{theo}
\label{theo:zhang}
If $x_i$ and $x_j$ are neither directly connected nor unconditionally independent, then there must exist a set of variables $Z$ and two functions $f$ and $g$ such that $x_i-f(Z)\indepe x_j-g(Z)$, and $x_i-f(Z)\indepe Z$ or $x_j-g(Z)\indepe Z$.
\end{theo}

In our case, $x_j \indepe x_i \mid (M_i \setminus \{ x_j \}) \Leftrightarrow x_j - f(M_i \setminus \{ x_j \}) \indepe x_i - g(M_i \setminus \{ x_j \})$, where $f$ and $g$ are multiple linear regression functions of $x_j$ on $M_i \setminus \{ x_j \}$ and $x_i$ on $M_i \setminus \{ x_j \}$, respectively. Since $(M_i \setminus \{ x_j \})\cap M_j= M_i \cap M_j$, we can assume that $x_j \indepe x_i \mid (M_i \setminus \{ x_j \}) \Leftrightarrow x_j - h(M_i \cap M_j) \indepe x_i - g(M_i \setminus \{ x_j \})$ where $h$ is a multiple linear regression function of $x_j$ on $(M_i \cap M_j)$.

Based on Theorem~\ref{theo:zhang}, RCD uses Lemma 4 to distinguish the parents from the ancestors. We proved Lemma 4 without using Theorem~\ref{theo:zhang}.

\begin{lem}
 Assume that $x_j \in M_i$; that is, $x_j$ is an ancestor of $x_i$. Let $z_i$ denote the residual obtained by the multiple regression of $x_i$ on $M_i \setminus \{ x_j \}$. Let $w_{j}$ denote the residual obtained by the multiple regression of $x_j$ on $(M_i \cap M_j)$. If $z_{i}$ and $w_{j}$ are linearly correlated, then $x_j$ is a parent of $x_i$; otherwise, $x_j$ is not a parent of $x_i$.
 \end{lem}
\begin{proof}
	Variable $x_{i}$ and $x_j$ are formulated as follows:
\begin{equation}
\label{equation:4-1}
	x_{i} = \sum_{x_m\in M_i} b_{im}x_{m} + \sum_n \lambda_{in}f_{n}+e_{i}
\end{equation}
\begin{equation}
\label{equation:4-2}
	x_{j} = \sum_{x_m\in M_j} b_{jm}x_{m} + \sum_n \lambda_{jn}f_{n}+e_{j}
\end{equation}
Let $\alpha_k$ denote the coefficient of $x_k\in (M_i\setminus \{x_j \} )$ when $x_i$ is regressed on $M_i\setminus \{x_j \}$. Then, 
\begin{align}
\label{equation:4-3}
	z_i &= x_i -\sum_{x_k\in (M_i\setminus \{x_j \})} \alpha_k x_{k}\nonumber\\
	&=\sum_{x_m\in M_i} b_{im}x_{m} + \sum_n \lambda_{in}f_{n}+e_{i} -\sum_{x_k\in (M_i\setminus \{x_j \})} \alpha_k x_{k}\nonumber\\
	&=b_{ij}x_j+\sum_{x_m\in (M_i\setminus \{x_j \})} b_{im}x_{m} + \sum_n \lambda_{in}f_{n}+e_{i} -\sum_{x_k\in (M_i\setminus \{x_j \})} \alpha_k x_{k}\nonumber\\
	&=b_{ij}\left( \sum_{x_m\in M_j} b_{jm}x_{m} + \sum_n \lambda_{jn}f_{n}+e_{j} \right)+\sum_{x_m\in (M_i\setminus \{x_j \})} b_{im}x_{m} + \sum_n \lambda_{in}f_{n}+e_{i}\nonumber\\
	&\ \ \ \ \ -\sum_{x_k\in (M_i\setminus \{x_j \})} \alpha_k x_{k}\nonumber\\
	&=b_{ij}\left( \sum_{x_m\in (M_j\setminus M_i)} b_{jm}x_{m} + \sum_{x_m\in (M_i\cap M_j)} b_{jm}x_{m} + \sum_n \lambda_{jn}f_{n}+e_{j} \right)+\sum_{x_m\in (M_i\setminus \{x_j \})} b_{im}x_{m} \nonumber\\
	&\ \ \ \ \ + \sum_n \lambda_{in}f_{n}+e_{i} -\sum_{x_k\in (M_i\setminus \{x_j \})} \alpha_k x_{k}\nonumber\\
	&=b_{ij}\left( \sum_{x_m\in (M_j\setminus M_i)} b_{jm}x_{m} + \sum_n \lambda_{jn}f_{n}+e_{j} \right) + b_{ij}\sum_{x_m\in (M_i\cap M_j)} b_{jm}x_{m}+\sum_{x_m\in (M_i\setminus \{x_j \})} b_{im}x_{m} \nonumber\\
	&\ \ \ \ \ + \sum_n \lambda_{in}f_{n}+e_{i} -\sum_{x_k\in (M_i\setminus \{x_j \})} \alpha_k x_{k}\nonumber\\
	&=b_{ij}\left( \sum_{x_m\in (M_j\setminus M_i)} b_{jm}x_{m} + \sum_n \lambda_{jn}f_{n}+ e_{j}\right) + \sum_n \lambda_{in}f_{n}+e_{i}
\end{align}
Let $\beta_k$ denote the coefficient of $x_k\in (M_i \cap M_j )$ when $x_j$ is regressed on $M_i \cap M_j$. Then, 
\begin{align}
\label{equation:4-4}
w_j &= x_j -\sum_{x_k\in (M_i \cap M_j)} \beta_k x_{k}\nonumber\\
&=\left(\sum_{x_m\in M_j} b_{jm}x_{m} + \sum_n \lambda_{jn}f_{n}+e_{j}\right) -\sum_{x_k\in (M_i \cap M_j)} \beta_k x_{k}\nonumber\\
&=\left(\sum_{x_m\in (M_j\setminus M_i)} b_{jm}x_{m} + \sum_{x_m\in (M_j\cap M_i)} b_{jm}x_{m} + \sum_n \lambda_{jn}f_{n}+e_{j}\right) -\sum_{x_k\in (M_i \cap M_j)} \beta_k x_{k}\nonumber\\
&= \sum_{x_m\in (M_j\setminus M_i)} b_{jm}x_{m} + \sum_n \lambda_{jn}f_{n}+ e_{j}
\end{align}
From Equations~\ref{equation:4-3}, and \ref{equation:4-4},
\begin{equation}
\label{equation:4-5}
	z_i = b_{ij} w_j + \sum_n \lambda_{in}f_{n}+e_{i}
\end{equation}
Since $x_i$ and $x_j$ do not have the same latent confounder:
\begin{equation}
\label{equation:4-6}
	\forall n, (\lambda_{in} = 0)\lor (\lambda_{jn} = 0)
\end{equation}
From Equations~\ref{equation:4-4}, \ref{equation:4-5}, and \ref{equation:4-6}, $z_i$ and $w_i$ are linearly correlated when $b_{ij}\neq 0$. It means that $x_j$ is a parent (direct cause) of $x_i$. When $b_{ij}= 0$, $z_i$ and $w_i$ are not linearly correlated. It means that $x_j$ is not a parent of $x_i$.
\end{proof}

\subsection{Identifying pairs of variables that have the same latent confounders}
RCD infers that two variables are affected by the same latent confounders when those two variables are linearly correlated even after removing the effects of all the parents. RCD identifies the pairs of variables affected by the same latent confounders by using Lemma 5. 

\begin{lem}
 Let $M_i$ and $M_j$ respectively denote the sets of ancestors of $x_i$ and $x_j$, and $P_i$ and $P_j$ respectively denote the sets of parents of $x_i$ and $x_j$. Assume that $x_i \notin M_j$ and $x_j \notin M_i$. Let $y_i$ denote the residual obtained by the multiple regression of $x_i$ on $P_i$, and $y_j$ denote the residual obtained by the multiple regression of $x_j$ on $P_j$. If $y_i$ and $y_j$ are linearly correlated, then $x_i$ and $x_j$ have the same latent confounders.
\end{lem}
\begin{proof}
	Variable $x_{i}$ and $x_j$ are formulated as follows:
\begin{equation}
\label{equation:5-1}
	x_{i} = \sum_{x_m\in P_i} b_{im}x_{m} + \sum_n \lambda_{in}f_{n}+e_{i}\nonumber
\end{equation}
\begin{equation}
\label{equation:5-2}
	x_{j} = \sum_{x_m\in P_j} b_{jm}x_{m} + \sum_n \lambda_{jn}f_{n}+e_{j}\nonumber
\end{equation}
Let $\alpha_k$ denote the coefficient of $x_k\in P_i$ when $x_i$ is regressed on $P_i$. Then, 
\begin{align}
	y_i &= x_i - \sum_{x_k\in P_i} \alpha_k x_{k}\nonumber\\
	&= \sum_{x_m\in P_i} b_{im}x_{m} + \sum_n \lambda_{in}f_{n}+e_{i} - \sum_{x_k\in P_i} \alpha_k x_{k}\nonumber\\
	&= \sum_n \lambda_{in}f_{n}+e_{i}\nonumber
\end{align}
Let $\beta_k$ denote the coefficient of $x_k\in P_j$ when $x_j$ is regressed on $P_j$. Then, 
\begin{align}
	y_j &= x_j -\sum_{x_k\in P_j} \beta_k x_{k}\nonumber\\
	&= \sum_{x_m\in P_j} b_{jm}x_{m} + \sum_n \lambda_{jn}f_{n}+e_{j} - \sum_{x_k\in P_j} \beta_k x_{k}\nonumber\\
	&= \sum_n \lambda_{jn}f_{n}+e_{j}\nonumber
\end{align}
Variables $e_i$ and $e_i$ are independent of each other. If we assume that $x_i$ and $x_j$ do not have the same latent confounder, then, 
\begin{equation}
	\forall n, (\lambda_{in} = 0)\lor (\lambda_{jn} = 0)\nonumber
\end{equation}
Then, $y_i$ and $y_j$ are mutually independent. However, this contradicts the assumption of Lemma 5 that $y_i$ and $y_j$ are linearly correlated. Therefore, $x_i$ and $x_j$ have the same latent confounders.
\end{proof}

\section{Performance evaluation}
\label{section:evaluation}
We evaluated the performance of RCD relative to the existing methods in terms of how accurately it finds the pairs of variables that are affected by the same latent confounders and how accurately it infers the causal directions of the pairs of variables that are not affected by the same latent confounder. In regard to the latent confounders, we compared RCD with FCI~\cite{fci}, RFCI~\cite{colombo2012}, and GFCI~\cite{ogarrio2016}. In addition to these three methods, we compared RCD with PC~\cite{Spirtes91}, GES~\cite{chickering2002}, DirectLiNGAM~\cite{shimizu2011}, and RESIT~\cite{peters2014} to evaluate the accuracy of causal directions. In the following sections, DirectLiNGAM is called LiNGAM for simplicity.

\subsection{Performance on simulated structures}

\begin{figure*}[h]
\centering
\includegraphics[width=12.0cm]{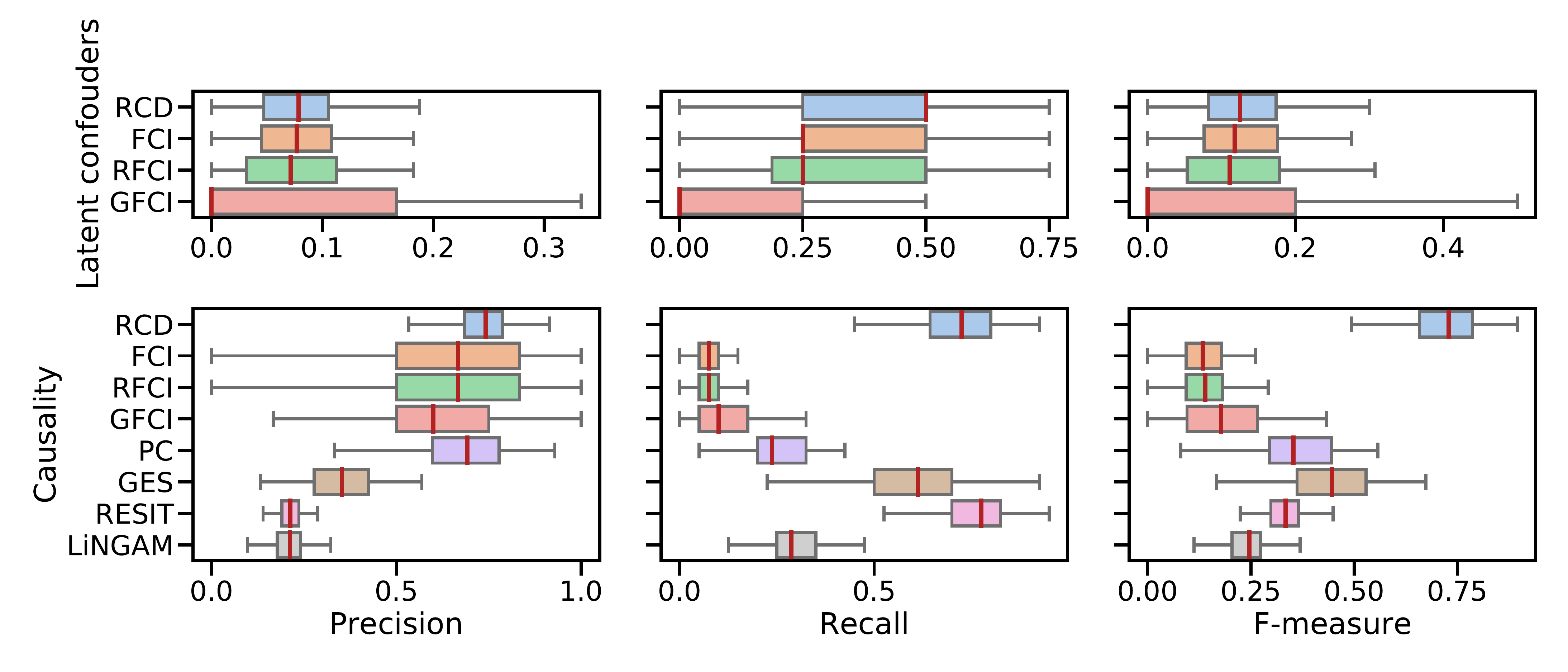}
\caption{Performance evaluation on causal graphs using simulated data: The vertical red lines indicate the median values of the results. The evaluation of the latent confounders corresponds to the evaluation of bi-directed arrows. The evaluation of causality corresponds to the evaluation of directed arrows.}
\label{figure:simulationresult}
\end{figure*}

We performed 100 experiments to evaluate RCD relative to the existing methods. We prepared 300 sets of samples for each experiment. The data of each experiment were generated as follows: The data generation process was modeled the same as Equation~\ref{equation:lingamco}. The number of observed variables $x_i$ was set to 20 and the number of latent confounders $f_k$ was set to 4. Let $X$ and $Y$ denote the stochastic variables, and assume that $Y\sim N(0.0,0.5)$ and $X=Y^3$. We used the random samples of $X$ for $e_i$ and $f_k$ because $X$ is non-Gaussian. The number of causal arrows between the observed variables is 40, and the start point and the end point of each causal arrow were randomly selected. We randomly drew two causal arrows from each latent confounder to the observed variables. Let $Z$ denote a stochastic variable that comes from a uniform distribution on $[-1.0,-0.5]$ and $[0.5,1.0]$. We used the random samples of $Z$ for $b_{ij}$ and $\lambda_{ik}$. \par
We evaluated (1) how accurately each method infers the pairs of variables that are affected by the same latent confounders (called the evaluation of latent confounders), and (2) how accurately each method infers causality between the observed variables that are not affected by the same latent confounder (called the evaluation of causality). The evaluation of latent confounders corresponds to the evaluation of bi-directed arrows in a causal graph, and the evaluation of causality corresponds to the evaluation of directed arrows. We used precision, recall, and F-measure as evaluation measures. In regard to the evaluation of latent confounders, true positive (TP) is the number of true bi-directed arrows that are correctly inferred. In regard to causality, TP is the number of true directed arrows that a method correctly infers in terms of their positions and directions. Precision is TP divided by the number of estimations, and recall is TP divided by the number of all true arrows. F-measure is defined as $\text{F-measure} = 2 \cdot \text{precision} \cdot \text{recall} / (\text{precision} + \text{recall})$.\par
The arguments of RCD, that is, $\alpha_{\text{C}}$ (alpha level for Pearson's correlation), $\alpha_{\text{I}}$ (alpha level for independence), $\alpha_{\text{S}}$ (alpha level for the Shapiro-Wilk test), and $n$ (maximum number of explanatory variables for multiple linear regression) were set as $\alpha_C=0.01, \alpha_I=0.01, \alpha_S=0.01,$ and $n=2$.\par
In regard to the types of edges, FCI, RFCI, and GFCI produce partial ancestral graphs (PAGs) that include six types of edges: $\rightarrow$ (directed), $\leftrightarrow$ (bi-directed), $\rightarrowCirc$ (partially directed), $\circHyphenCirc$ (nondirected), and $\circHyphen$ (partially undirected). In the evaluation, we only used the directed and bi-directed edges. PC, GES, LiNGAM, and RESIT produce causal graphs only with the directed edges; thus, we did not evaluate those methods in terms of latent confounders.\par
The box plots in Figure~\ref{figure:simulationresult} display the results. The vertical red lines indicate the median values. Note that some median values are the same as the upper or lower quartiles. For example, the median and the upper quartile of the recalls of RCD in the results of latent confounders are the same. It means that the results between the median and the upper quartile are the same. In regard to the evaluation of latent confounders, the precision, recall, and F-measure values are almost the same for RCD, FCI, RFCI, and GFCI, but the medians of precision, recall, and F-measure values of RCD are the highest among them. In regard to causality, RCD scores the highest medians of the precision and F-measure values among all the methods, and the median of recall for RCD is the second highest next to RESIT.\par
The results suggest that RCD does not greatly improve the performance metrics compared to the existing methods. However, there is no other method that has the highest or the second highest performance for each metric. FCI, RFCI, and GFCI perform as well as RCD in terms of finding the pairs of variables that are affected by the same latent confounders, but they do not perform well in terms of the recall of causality. In addition, no other method performs well in terms of both precision and recall of causality. RCD can successfully find the pairs of variables that are affected by the same latent confounders and identify the causal direction between variables that are not affected by the same latent confounder.

\subsection{Performance on real-world structures}

Causal structures in the real-world are often very complex. Therefore, RCD likely produces a causal graph where each pair of observed variables is connected with a bi-directed arrow. The result of identifying latent confounders is affected by the threshold of the p-value for the independence test, $\alpha_I$. If $\alpha_I$ is too large or too small, then all the variable pairs are likely concluded to have the same latent confounders. Therefore, we need to find the most appropriate value of $\alpha_I$. We increased $k$ from 1 to 25 and set $\alpha_I$ as $\alpha_I = 0.1^{k}$ and repeated the process. We adopted a result that has the smallest number of pairs of variables with the same latent confounders.\par

\begin{figure}[t]
\centering
\includegraphics[width=7.0cm]{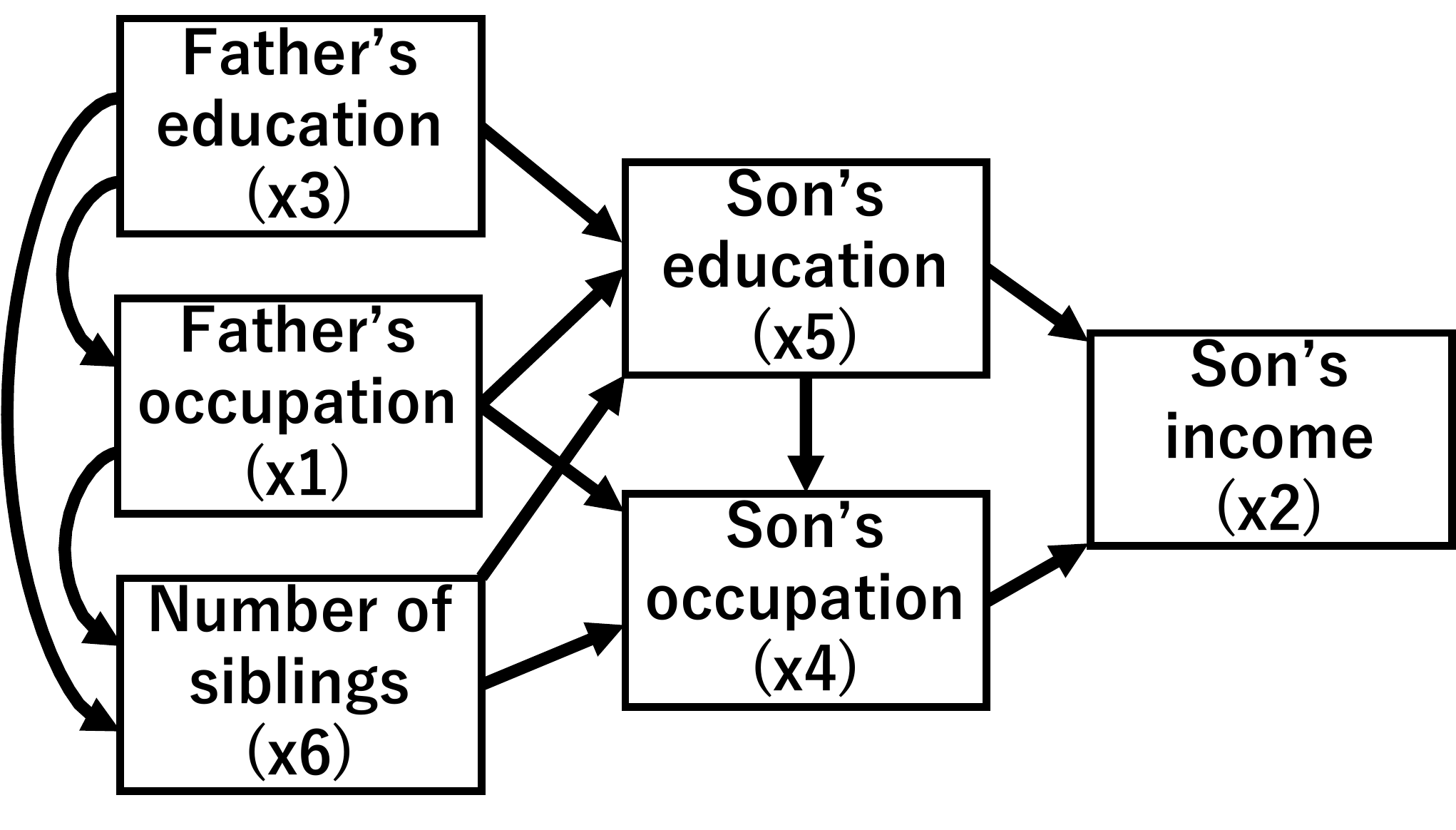}
\caption{Variables and causal relations in the General Social Survey data set used for the evaluation.}
\label{figure:socio}
\end{figure}

\begin{table*}[t]
\footnotesize
\caption{The results of the application to sociological data.}
\label{table:socioresult}
\begin{center}
\begin{tabular}{|c|c|c|c|c|c|c|}\hline
& \multicolumn{3}{|c|}{Bidirected arrows (Latent confounders)}  & \multicolumn{3}{|c|}{Directed arrows (Causality)} \\\cline{2-7}
Method & \# of estimation & \# of successes & Precision & \# of estimation & \# of successes & Precision\\\hline
RCD & 4 & 4 & 1.0  & 5 & 4 & 0.8 \\
FCI & 3 & 3 & 1.0  & 3 & 1 & 0.3 \\
RFCI & 3 & 3 & 1.0  & 3 & 1 & 0.3 \\
GFCI & 0 & 0 & 0.0  & 0 & 0 & 0.0 \\
PC & - & - & - & 2 & 1 & 0.5 \\
GES & - & - & - & 2 & 1 & 0.5 \\
RESIT & - & - & - & 12 & 4 & 0.3 \\
LiNGAM & - & - & - & 5 & 4 & 0.8 \\\hline

\end{tabular}
\end{center}
\end{table*}

We analyzed the General Social Survey data set, taken from a sociological data repository.\footnote{http://www.norc.org/GSS+Website/} The data have been used for the evaluation of DirectLiNGAM in Shimizu et al.~\cite{shimizu2011}. The sample size is 1380. The variables and the possible directions are shown in Figure~\ref{figure:socio}. The directions were determined based on the domain knowledge in Duncan et al.~\cite{duncan1972} and temporal orders.\par

We evaluated the directed arrows (causality) in the causal graphs produced by RCD and the existing methods, based on the directed arrows in Figure~\ref{figure:socio}. In addition, we evaluated the bi-directed arrows in causal graphs produced by the methods as accurate inference if they exist in Figure~\ref{figure:socio} as directed arrows.\par
The results are listed in Table~\ref{table:socioresult}. In regard to bi-directed arrows (latent confounders), the number of successful inferences by RCD is the highest, and the precisions of RCD, FCI, and RFCI are all 1.0. In regard to the directed arrows (causality), the numbers of the successful arrows of RCD, RESIT, and LiNGAM are the highest. The precisions of RCD and LiNGAM are also the highest.
The causal graph produced by RCD is shown in Figure~\ref{figure:sociorcd}. The dashed arrow $x_3 \leftarrow x_5$ is the incorrect inference, but the others are correct.\par
RCD performs the best among the existing methods in terms of both identifying the pairs of variables that are affected by the same latent confounders and identifying the causal direction of the pairs of variables that are not affected by the same latent confounder.

\begin{figure}[t]
\centering
\includegraphics[width=7.0cm]{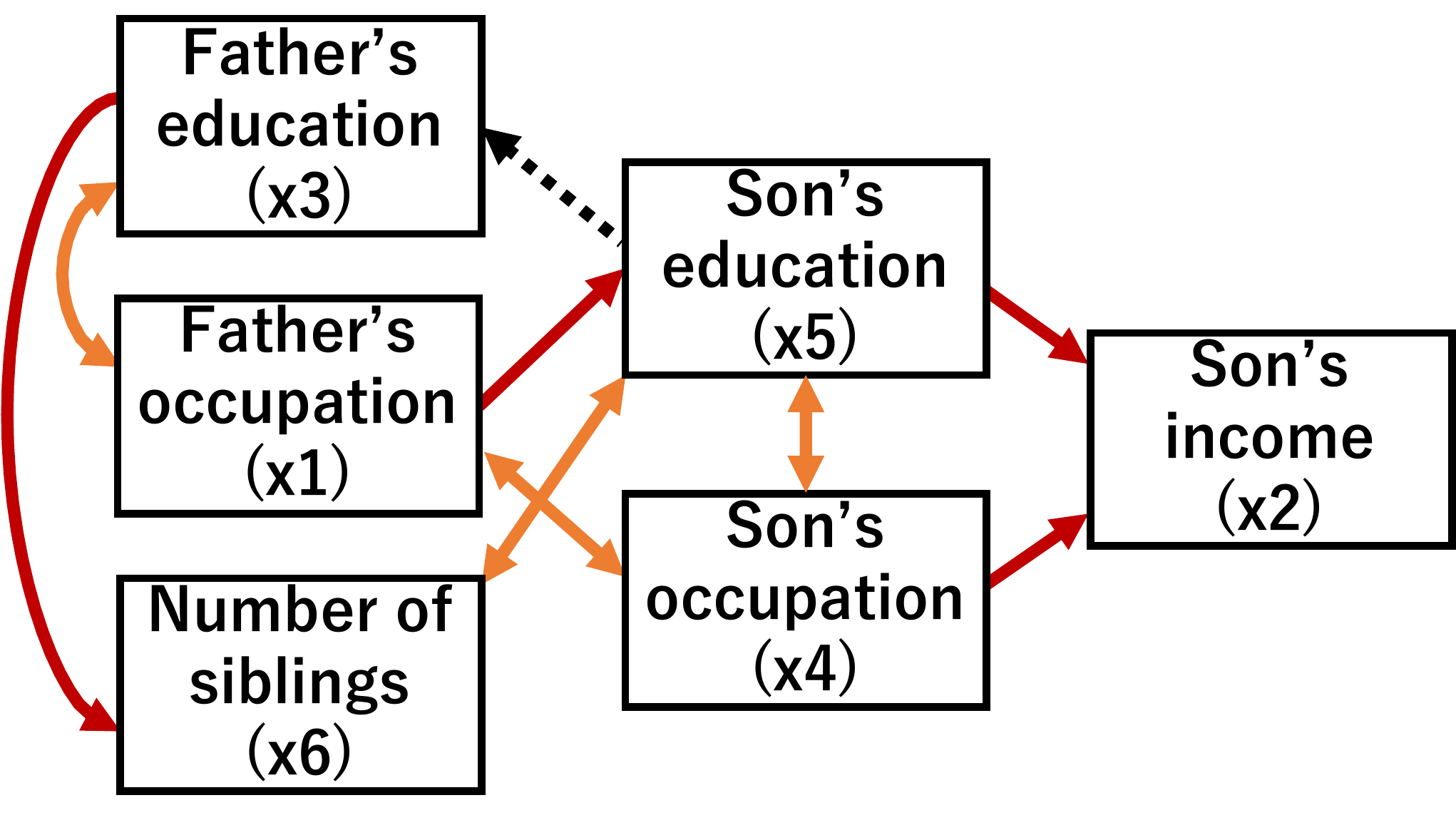}
\caption{Causal graph produced by RCD: The dashed arrow, $x_3 \leftarrow x_5$ is incorrect inference, but the other arrows are reasonable based on Figure~\ref{figure:socio}}
\label{figure:sociorcd}
\end{figure}

\section{Conclusion}
\label{section:conclusion}
We developed a method called repetitive causal discovery (RCD) that produces a causal graph where a directed arrow indicates the causal direction between the observed variables, and a bi-directed arrow indicates a pair of variables have the same confounder. RCD produces a causal graph by (1) finding the ancestors of each variable, (2) distinguishing the parents from the indirect causes, and (3) identifying the pairs of variables that have the same latent confounders. We confirmed that RCD effectively analyzes data confounded by unobserved variables through validations using simulated and real-world data.\par
In this paper, we did not discuss the utilization of prior knowledge. However, it is possible to make use of prior knowledge of causal relations in practical applications of RCD. In this study, information about the ancestors of each variable was initialized to be an empty set. If we have prior knowledge about causal relations, the information about the ancestors of each variable that RCD retains can be set according to the prior knowledge.\par
There is still room for improvement in the RCD method. The optimal settings of the arguments of RCD and the extension of RCD for nonlinear causal relations will be investigated in future studies.

\section{Acknowledgments}
We thank Dr. Samuel Y. Wang for his useful comments on a previous version of our algorithm proposed in \cite{maeda20a}.
Takashi Nicholas Maeda has been partially supported by Grant-in-Aid for Scientific Research (C) from Japan Society for the Promotion of Science (JSPS) \#20K19872. Shohei Shimizu has been partially supported by ONRG NICOP N62909-17-1-2034 and Grant-in-Aid for Scientific Research (C) from Japan Society for the Promotion of Science (JSPS) \#16K00045 and \#20K11708. 

\section*{References}

\bibliography{ref}

\end{document}